\documentclass[11pt,letterpaper]{amsart}
\usepackage{amssymb}
\usepackage{amsmath}
\usepackage{amsthm}
\usepackage{bbm}
\usepackage{mathtools}
\usepackage{graphicx}
\usepackage{verbatim}
\usepackage{fullpage}
\usepackage{yfonts}
\usepackage[T1]{fontenc}
\usepackage[utf8]{luainputenc}
\usepackage{appendix}
\usepackage{hyperref}
\usepackage{tikz}
\usepackage{tikz-cd} 

\usepackage[backend=biber,
style=numeric,
citestyle=numeric,
giveninits=true,
minbibnames=1,
maxbibnames=1
]{biblatex}

\def\reals{\mathbb{R}}

\def\tr{\mathrm{tr}\;}

\newtheorem{definition}{Definition}
\newtheorem{theorem}{Theorem}
\newtheorem{lemma}{Lemma}
\newtheorem{proposition}{Proposition}

\newtheorem{corollary}{Corollary}
\newtheorem{remark}{Remark}
\newtheorem{example}{Example}

\renewcommand{\vec}[1]{\mathbf{#1}}
\newcommand{\trans}{\intercal}
\newcommand{\diag}{\mathrm{diag}}

\DeclareMathOperator{\kernel}{Ker}

\DeclareMathOperator{\vspan}{span}

\DeclareMathOperator*{\argmin}{arg\,min}
\DeclareMathOperator*{\rank}{rank}

\usepackage{algorithm}
\usepackage{algpseudocode}

\makeatletter
\newcommand*{\algrule}[1][\algorithmicindent]{\makebox[#1][l]{\hspace*{.5em}\vrule height .75\baselineskip depth .25\baselineskip}}%

\newcount\ALG@printindent@tempcnta
\def\ALG@printindent{%
    \ifnum \theALG@nested>0
    \ifx\ALG@text\ALG@x@notext
    \addvspace{-3pt}
    \else
    \unskip
    \ALG@printindent@tempcnta=1
    \loop
    \algrule[\csname ALG@ind@\the\ALG@printindent@tempcnta\endcsname]%
    \advance \ALG@printindent@tempcnta 1
    \ifnum \ALG@printindent@tempcnta<\numexpr\theALG@nested+1\relax
    \repeat
    \fi
    \fi
}%
\usepackage{etoolbox}
\patchcmd{\ALG@doentity}{\noindent\hskip\ALG@tlm}{\ALG@printindent}{}{\errmessage{failed to patch}}
\makeatother
\makeatother

\title{A proof of convergence of multi-class logistic regression network}
\author{Marek Rychlik\\
  University of Arizona\\
  Department of Mathematics, 617 N Santa Rita Rd, P.O. Box 210089\\
  Tucson, AZ 85721-0089, USA\\}

\date{\today}

\date{\today}

\subjclass[2010]{%
  92B20, %
  68T05, %
  82C32%
}

\begin{document}
\maketitle
\begin{abstract}
  This paper revisits the special type of a neural network known under
  two names. In the statistics and machine learning community it is
  known as a multi-class logistic regression neural network. In the
  neural network community, it is simply the soft-max layer. The
  importance is underscored by its role in deep learning: as the last
  layer, whose autput is actually the classification of the input
  patterns, such as images. Our exposition focuses on mathematically
  rigorous derivation of the key equation expressing the gradient. The
  fringe benefit of our approach is a fully vectorized expression,
  which is a basis of an efficient implementation. The second result
  of this paper is the positivity of the second derivative of the
  cross-entropy loss function as function of the weights. This result
  proves that optimization methods based on convexity may be used to
  train this network. As a corollary, we demonstrate that no
  $L^2$-regularizer is needed to guarantee convergence of gradient
  descent, provided that a global minimum of the loss function exists.
  We also provide an effective bound on the rate of convergence for
  two classes.
\end{abstract}

\renewcommand{\include}[1]{\input{#1}}

\section{Notation and definitions}
The multi-class logistic regression network is a neural network
which takes an input vector $\vec{x}\in\reals^D$ and produces
an activation vector $\vec{a}\in\reals^C$ by a linear
transformation
\[ \vec{a} = \vec{W}\,\vec{x},\]
where $\vec{W}=[w_{jk}]$ is an $C\times D$ matrix of weights.  The
vector space of such matrices will be denoted by
$L\left(\reals^D,\reals^C\right)$ and identified with the space of
linear transformations \[\vec{W}: \reals^D\to\reals^C.\]

The activations are subsequently transformed by the \emph{soft-max function}
$\boldsymbol\sigma:\reals^C\to\reals^C$ given by the formula
\begin{equation}
  \label{eqn:sigma}
  \sigma_i(\vec{u}) = \frac{e^{u_i}}{\sum_{j=1}^C e^{u_j}}.
\end{equation}
Thus,
\[ \vec{y} = \boldsymbol\sigma\left( \vec{W}\,\vec{x} \right) \] and
clearly $\vec{y}\in\reals^C$. The number $C$ is the number of classes
into which the input vectors will be classified into. For a
\emph{training vector} $\vec{x}$ its classification is known and it is
given by a target vector $\vec{t}\in\{0,1\}^C \subset\reals^C$, where
$t_k=1$ iff vector $\vec{x}$ belongs to class $k$.  This kind of
encoding of classes, or categories, by vectors of the standard basis
of $\reals^C$ is called \emph{one-hot encoding}.  Thus
\[ \sum_{k=1}^C t_k = 1. \]
In order to train the network, we need the \emph{training set}
consisting of $N$ sample input vectors
\[ \vec{x}^{(n)}\in\reals^D,\qquad n=1,2,\ldots, N \]
and $N$ corresponding target vectors
\[ \vec{t}^{(n)} \in \reals^C. \]

The loss function to be minimized is the \emph{cross-entropy} loss function,
given by the formula:
\begin{equation}
  \label{eqn:loss-function}
  L(\vec{W}) = -\sum_{n=1}^N \sum_{i=1}^C t_i^{(n)} \log y_i^{(n)}.
\end{equation}
The problem of training is the problem of finding the optimal weight matrix
$\widehat{\vec{W}}$ which minimizes $L(\vec{W})$:
\[ \widehat{\vec{W}} = \argmin_{\vec{W}\in L(\reals^D,\reals^C)} L(\vec{W}). \]

The natural question arises, whether the minimum exists and whether it is
unique. It is known that the simplistic answer to this question is ``no''
because shifting the weights by a constant depending on $n$ only
does not change the value of $L$. More precisely, if
\[\tilde{\vec{W}} = \vec{W} + \mathbbm{1}\,\vec{c}^\trans\]
then $L(\tilde{\vec{W}}) = L(\vec{W})$, and, more strongly
\[ \tilde{\vec{y}}^{(n)} = \vec{y}^{(n)}\]
where
\begin{align*}
  \vec{y}^{(n)} &= \boldsymbol\sigma\left(\vec{W}\,\vec{x}^{(n)}\right).\\
  \tilde{\vec{y}}^{(n)} &= \boldsymbol\sigma\left(\tilde{\vec{W}}\,\vec{x}^{(n)}\right).  
\end{align*}
This is the consequence of the following identity: for every $c\in\reals$
\[\boldsymbol\sigma(\vec{u} + c\, \mathbbm{1}) = \boldsymbol\sigma(\vec{u}).\]

This brief paper answers this question in the most satisfactory
fashion, giving sufficient conditions for $L$ to be a strictly convex function
on the subspace of those weight matrices $\vec{W}$ for which the sum of every
column is $0$, i.e.
\[ \mathbbm{1}^\trans\, \vec{W} = 0. \]
This condition guarantees the uniqueness of the global minimum (subject to the condition above),
and convergence of optimization algorithms which depend on strict convexity,
but \textbf{does not guarantee the existence} of a global minimum. 
\begin{example}
  It is easy to construct a strictly convex function $L:\reals\to\reals$ which
  does not have a global minimum, e.g.
  \[L(x) = \log(\cosh x)+2\,x \]
  We note that $L'(x)=\tanh x + 2 \in (1,3)$. Hence $L$ is strictly increasing.
\end{example}
One way to choose the matrix $\vec{W}$ is to shift every column by
subtracting the mean, which will have the effect of choosing the
equivalent weight matrix with miniumum Frobenius norm:
\[\sum_{j=1}^C\sum_{k=1}^D w_{jk}^2 = \min.\]
A practical application of the above observation is that we can subtract
the mean from columns as part of the iteration process, when implementing
an optimization method, such as gradient descent, which should help with
stability of the method.

\subsection{Related reading}
Many related techniques are described in
\cite{logistic-regression-wikipedia,multinomial-logistic-regression-wikipedia}.
The softmax function use for pattern recognition is also discussed in
\cite{bishop-book} on pp.~215 and pp.~238--245, and in
\cite{murphy-book} on pp.~252--254.


\section{On the meaning of loss}
In this short section we relate probability theory, statistics and
machine learning, hoping that this will be helpful to some readers.
We will follow the interpretation of the loss function is as negative
log-likelihood, using conventional statistical modeling
assumptions. The fundamental assumption is that the neural network is
capable of outputting exact (conditional) probabilities of certain
events conditioned upon the observed data and the parameters of the
model.

In order to give our considerations a solid, formal, probability space
we will assume the existence of a number of random variables on a probability space
$(\Omega,\Sigma,P)$. Thus, we assume that each of our samples
$\vec{x}^{(n)}$ is a value of a vector-value random variable
\[\vec{X}^{(n)}:\Omega\to\reals^D.\]
i.e. $\vec{x}^{(n)} = \vec{X}^{(n)}(\omega)$, where $\omega$ is the
outcome (elementary event) associated with our experiment. We also
assume that there are (scalar and discrete) random variables
\[L^{(n)}: \Omega\to \{1,2,\ldots,C\}\subseteq\reals\] which assign
labels, i.e. classes, to the samples. We further assume that for a
fixed sequence of labels $\ell^{(n)}$, $n=1,2,\ldots,N$, the joint
probability distribution of random variables $L^{(n)}$ is given by:
\[ P\left(L^{(n)}(\omega)=\ell^{(n)}\;\text{for $n=1,2,\ldots,N$}
    \,\bigg|\, \mathcal{X},\vec{W} \right) =
  \prod_{n=1}^N y_{\ell^{(n)}}^{(n)}.
\]
where for $\ell=1,2,\ldots,C$
\[ y_\ell^{(n)} = \sigma_\ell\left(\vec{W}\,\vec{x}^{(n)}\right) \]
is the output of our neural network, and $\mathcal{X}$ represents our data:
\[ \mathcal{X} = \{ \vec{x}^{(n)}\,:\, n=1,2,\ldots,N \}. \]

It should be noted that we \textbf{do not} assume that $\vec{X}^{(n)}$
are independent random variables, and neither we assume that they are
identically distributed. We only assume that the labeling variables
$L^{(n)}$ are \textbf{conditionally independent}, given the
sample, i.e. the values of the random variables $\vec{X}^{(n)}$. The
conditioning on the data $\mathcal{X}$ is equivalent to the
conditioning on the $\sigma$-algebra generated by random variables
$\vec{X}^{(n)}$, $n=1,2,\ldots,N$, usually denoted
$\sigma\left(\vec{X}^{(1)},\vec{X}^{(2)},\ldots,\vec{X}^{(N)}\right)$,
with a bit of confusion related to the occurrence of $\sigma$ in
multiple contexts.

It is also easy to define random variables modeling the one-hot
encoded labels. Thus, we define random variables
$T_\ell^{(n)}:\Omega\to\{0,1\}\subseteq\reals$, $\ell=1,2,\ldots,C$,
by:
\[ T_\ell^{(n)} = 1 \iff L^{(n)} = \ell.\]
Hence, we also have vector-value random variables assigning the one-hot encoded labels
\[ \vec{T}^{(n)} : \Omega \to \{0,1\}^C \subseteq\reals^C \]
whose scalar components are $T_i^{(n)}$. By definition
\[ y_\ell^{(n)} = P\left( T_\ell^{(n)} = t_\ell^{(n)} \,\big|\, \mathcal{X}, \vec{W} \right). \]
It follows easily that the joint distribution of $\vec{T}^{(n)}$, $n=1,2,\ldots,N$, is given by:
\begin{equation}
  \label{eqn:probability-model}
  P\left(\vec{T}^{(n)}(\omega) =\vec{t}^{(n)}\;\text{for $n=1,2,\ldots,N$}
    \,\bigg|\, \mathcal{X},\vec{W} \right) =
  \prod_{n=1}^N \prod_{\ell=1}^C \left(y_\ell^{(n)}\right)^{t_\ell^{(n)}}.
\end{equation}
The left-hand side is a function of the sample, i.e. it is a statistic. Given the above completely
rigorous probability model, we obtain formula \eqref{eqn:loss-function} by
taking the negative logarithm of \eqref{eqn:probability-model}. Thus:
\[ L(\vec{W}) = -\log P\left(\vec{T}^{(n)}(\omega)=\vec{t}^{(n)}\;\text{for $n=1,2,\ldots,N$}
    \,\bigg|\, \mathcal{X},\vec{W} \right). \] 
It should be noted that likelihood is not defined in terms of the above conditional probability
but rather the probability of \emph{all observed data}:
\[\mathcal{D} = \left\{\mathcal{X},\mathcal{L}\right\}\]
where
\[ \mathcal{L} = \left\{\ell^{(n)}\,:\,n=1,2,\ldots,N\right\}\] are the (known)
labels of our data. Thus the likelihood is the joint, conditional
distribution:
\[ P\left(\mathcal{X}, \mathcal{L}\,\big|\, \vec{W} \right). \]
However,
\[
  P\left(\mathcal{X}, \mathcal{L} \big|\,\vec{W}\right)
  = P\left(\mathcal{L}\,\big|\,\mathcal{X},\vec{W}\right)\,
  P\left(\mathcal{X}\,\big|\,\vec{W}\right).
\]
where $P\left(\mathcal{X}\,\big|\,\vec{W}\right)$ is the conditional
distribution of data for given weights. It is reasonable to assume
that data is independent of the weights, as the process of drawing
samples $\mathcal{X}$ and determining true labels $\mathcal{L}$ is
unrelated to picking weights $\vec{W}$ and hence
$ P\left(\mathcal{X}\,\big|\,\vec{W}\right) =
P\left(\mathcal{X}\right)$.  There the loss function we study can be
written, up to an additive term $-\log P(\mathcal{X})$ which does not
depend on $\vec{W}$, as:
\[ L(\vec{W}) = -\log P\left(\mathcal{L}\,\big|\,\mathcal{X},\vec{W}\right). \]
Thus, minimizing $L(\vec{W})$ is equivalent to picking weights which
maximize the probability of all observed data.

Justifying the term \emph{cross-entropy} for the loss function would
require a deeper dive into information theory, especially in view of
the fact that its use in the current context is more traditional than
correct. The analogous mistake would be to confuse sample means with
means (i.e. expected values) of random variables.

Let us comment on the use of the $L^2$ regularizer. This is a modification
of $L(\vec{W})$ in the following way:
\[ \tilde{L}(\vec{W}) = L(\vec{W}) + \dfrac{1}{2}\alpha\|\vec{W}\|^2 \]
where $\alpha > 0$. Let us assume that we draw the weights $\vec{W}$
from a normal distribution, more precisely:
\[ -\log P(\vec{W}) = \dfrac{1}{2}\alpha \|\vec{W}\|^2.\]
Rigorously we can establish this formula for the joint distribution
of $\mathcal{X}$, $\mathcal{L}$ and $\vec{W}$:
\[ P(\mathcal{X},\mathcal{L},\vec{W}) =
  P\left(\mathcal{L}\,\big|\,\mathcal{X},\vec{W}\right)\,
  P\left(\mathcal{X}\,\big|\,\vec{W}\right)\,P(\vec{W}) \]
Assuming as before that $P(\mathcal{X}\,\big|\,\vec{W}) = P(\mathcal{X})$ we arrive at:
\[ \tilde L(\vec{W}) = - \log P(\mathcal{L}\,\big|\,\mathcal{X},\vec{W}) + \log P(\vec{W})
  = -\log P(\mathcal{X},\mathcal{L},\mathcal{W}) \]
Thus, we have proven that the modified loss function satisfies:
\[ \tilde L(\vec{W}) = - \log P(\mathcal{X},\mathcal{L},\mathcal{W}). \]
Thus $\tilde{L}(\vec{W})$ no longer expresses the desired likelihood!
Hence, adding the regularizer is not mathematically correct. We don't
think it should be added merely for increased numerical stability of
the gradient method, at the expense of converging to incorrect value
of $\vec{W}$.  In the current paper we will focus on the
mathematically correct formulation and will see where it leads us.  We
would also like to bring to the readers's attention that in our
derivations we did not make use of the Bayes formula.


\section{The gradient and critical points of $L$}
\label{section:critical-points}
In this section we summarize known conditions for $\vec{W}$ to be a
critical points of $L$. The emphasis of our exposition is on
coordinate-free formulations, which naturally yields vectorized
formulas for the relevant quantities, which is of practical
significance in computation, and helps understanding some theoretical
aspects of the method.
\subsection{The gradient for a single sample}
To avoid the complexity of the notation of the general case, we assume
$N=1$, i.e. our dataset consists of a single sample.  This allows us
to avoid the superscript $(n)$. It will then be easy to utilize the
results for arbitrary sample sizes. Thus, in this section we consider
the truncated loss function:
\begin{equation}
  \label{eqn:truncated-loss-function}
  L(\vec{W}) = -\sum_{i=1}^C t_i \log y_i.
\end{equation}
where $\vec{y}\in\reals^C$ depends on $\vec{W}$. The summation over all
samples yields the result for the entire training set. The manner in
which we derive the formulas may differ from a typical derivation in
machine learning texts, in that it relies upon the Chain Rule for
\emph{Fr\'echet derivatives}, rather than calculation of partials with
respect of individual weights $w_{jk}$.\footnote{It should be remembered that
Fr\'echet derivatives are closely related to \emph{Jacobi matrices},
but they are \textbf{not the same}, as it will be quite apparent in
our calculations. Fr\'echet derivative is a linear transformation, and
upon the choice of a basis, or, in the case of $\reals^n$, when the standard
basis is used, this linear transformation is identified with the Jacobi matrix.}

With this approach, we automatically derive \emph{vectorized} formulas
for the gradient of $L$, which subsequently leads to very efficient
implementation of the training algorithm, as all critical operations
are simply matrix products.

We also allow the target vector $\vec{t}\in\reals^C$, without necessarily
requiring that $t_k\in\{0,1\}$, but we still require
\[ \sum_{k=1}^C t_k = 1. \]
From the point of view of applications, this generalization is useful when
the classification of the inputs is ambiguous. For instance, we can
use several humans to classify the inputs, in which case the humans may classify
the inputs differently. Then we could assign to $t_k$ the fraction of humans
who assign the input to class $k$.

We use the following representation of the loss function $L$:
\[ L(\vec{W}) = Q_{\vec{t}}(-\boldsymbol\log(\boldsymbol\sigma( P_{\vec{x}} (\vec{W}) ) ) )\]
where $\boldsymbol\log$ is the vectorized version of $\log$ (coordinatewise $\log$) and
\begin{align*}
  P_{\vec{x}}(\vec{W}) &= \vec{W}\,\vec{x}, \\
  Q_{\vec{t}}(\vec{z}) &= \vec{t}^\trans\,\vec{z}.
\end{align*}
is a vector-valued \emph{linear} operator \emph{on matrices}, consisting
in multiplying a matrix by $\vec{x}$ on the right. Thus
\[ L = Q_{\vec{t}}\circ ((-\boldsymbol\log) \circ \boldsymbol\sigma) \circ P_\vec{x}, \]
which is a composition involving linear operators $Q_{\vec{t}}$ and $P_{\vec{x}}$
and two non-linear transformations, $\boldsymbol\log$ and
$\boldsymbol\sigma$. Let $\boldsymbol\rho = (-\boldsymbol\log)\circ \boldsymbol\sigma$.
It will be beneficial to think of $L$ as
\[ L = Q_{\vec{t}}\circ \boldsymbol\rho \circ P_{\vec{x}} \]
which is a composition with only one non-linear term.
\begin{remark}
  This form is particularly useful for calculating higher derivatives
  of $L$, specifically, the Hessian of $L$.
\end{remark}
The Chain Rule yields the Fr\'echet derivative
of the composition:
\[DL  = DQ_{\vec{t}}\, D\boldsymbol\rho\, DP_{\vec{x}} \]
where the intermediate Fr\'echet derivatives are evaluated at respective
intermediate values of the composition. The more detailed version of the
formula is:
\[DL  = \left(DQ_{\vec{t}}\circ\boldsymbol\rho\circ P_{\vec{x}}\right)\, \left(D\boldsymbol\rho\circ P_{\vec{x}}\right)\, DP_{\vec{x}} \]
It should be noted that for a function $\vec{F}:\reals^n\to\reals^m$,
\[ D\vec{F}: \reals^n \to L(\reals^n,\reals^m), \]
i.e. it is an operator-valued function on $\reals^n$.

The Fr\'echet derivative of a
linear transformation is the transformation itself. Thus
\begin{align}
DQ_{\vec{t}} &= Q_{\vec{t}}\\
DP_{\vec{x}} &= P_{\vec{x}}
\end{align}
regardless of the argument. 

We state one more useful formula. For any vectorized scalar function,
evaluated elementwise, i.e. function $\vec{F}:\reals^n\to\reals^n$
given by:
\[ \vec{F}(\vec{x}) = (f(x_1),f(x_2),\ldots,f(x_n)) \]
is a diagonal matrix:
\[ D\vec{F}(\vec{x}) = \begin{bmatrix}
    f'(x_1) & 0 & \ldots & 0 \\
    0       & f'(x_2) & \ldots & 0 \\
    \vdots  & \vdots & \ddots & \vdots \\
    0       & 0      &  \ldots     & f'(x_n)
  \end{bmatrix}
  = \diag(f'(\vec{x})).
\]
where $\diag$ is a MATLAB-like operator converting a vector to a diagonal
matrix. We make an observation that $\diag$ itself is a linear operator
from vectors to matrices.

In summary,
\begin{equation}
  \label{eqn:derivative-formula}
  DL(\vec{W}) = Q_{\vec{t}} \, D\boldsymbol\rho(\vec{a})\, P_{\vec{x}}
\end{equation}
Let us compute the derivative of $\boldsymbol\rho$. Firstly,
\begin{equation}
  \label{eqn:rho-formula}
  \boldsymbol\rho(\vec{u}) = -\vec{u} + \left(\log\left(\sum_{k=1}^C
      e^{u_i}\right)\right)\,\mathbbm{1}
\end{equation}
where $\mathbbm{1}=(1,1,\ldots,1)$ is a (column) vector of 1's.  This
implies easily, using a combination of techniques already mentioned:
\begin{align*}
  D\boldsymbol\rho(\vec{u}) \vec{h}
  &= -\vec{h} + \left( \frac{1}{\sum_{k=1}^C e^{u_i}} \left(e^{\vec{u}}\right)^\trans\,\vec{h} \right) \mathbbm{1} \\
  &= -\vec{h} + (\boldsymbol\sigma(\vec{u})^\trans\,\vec{h})\mathbbm{1} 
  = -\vec{h} + \mathbbm{1}\,(\boldsymbol\sigma(\vec{u})^\trans\,\vec{h})
  = \left(-\vec{I} + \mathbbm{1}\,\boldsymbol\sigma(\vec{u})^\trans\right)\vec{h}    
\end{align*}
where $\vec{I}$ is the $C\times C$ identity matrix. Thus, in short,
\begin{equation}
  \label{eqn:rho-derivative}
  D\boldsymbol\rho(\vec{u}) = -\vec{I} + \mathbbm{1}\,\boldsymbol\sigma(\vec{u})^\trans.
\end{equation}

Finally, carefully looking
at the formula \eqref{eqn:derivative-formula} and equating matrix product
with composition of linear transformations (this convention should be familiar from
linear algebra), we obtain
\begin{align*}
  \label{eqn:truncated-gradient}
  DL(\vec{W})\vec{V}
  &= \vec{t}^\trans\,\left(-\vec{I} + \mathbbm{1}\,\boldsymbol\sigma(\vec{W}\,\vec{x})^\trans\right)\,\vec{V}\,\vec{x}\\
  &= -\vec{t}^\trans\,\vec{V}\,\vec{x} + (\vec{t}^\trans\,\mathbbm{1})\,\vec{y}^\trans\,\vec{V}\,\vec{x} \\
  &= -\vec{t}^\trans\,\vec{V}\,\vec{x} + \vec{y}^\trans\,\vec{V}\,\vec{x} \\
  &= -(\vec{t}-\vec{y})^\trans\,\vec{V}\,\vec{x}.  
\end{align*}
We note that we used $\vec{t}^\trans \mathbbm{1} = 1$ because $\sum_{k=1}^Ct_k = 1$ was assumed.

In summary, the derivative of $L$ for $N=1$ admits this simple equation
\begin{equation}
  \label{eqn:truncated-gradient}
  DL(\vec{W})\vec{V} = -(\vec{t}-\vec{y})^\trans\,\vec{V}\,\vec{x}.  
\end{equation}

\subsection{The gradient for arbitrary sample size}
Formula~\eqref{eqn:truncated-gradient} generalizes easily to $N\ge 1$,
when $L$ is given by \eqref{eqn:loss-function}:
\[ DL(\vec{W}) \vec{V} =
  -\sum_{n=1}^N \left(\vec{t}^{(n)}-\vec{y}^{(n)}\right)^\trans \vec{V} \vec{x}^{(n)}.
\]

In order to use methods such as gradient descent, we need to find the
gradient $\nabla L(\vec{W})$ from the formula for $DL(\vec{W})$.
We should note that the gradient belongs to the same vector space
as the argument $\vec{W}$, while $DL(\vec{W})$ is a functional
on the same space. In our situation:
\begin{align}
  \nabla L(\vec{W}) &\in L(\reals^D,\reals^C)\\
  DL(\vec{W}) &\in L(L(\reals^D,\reals^C),\reals) = L(\reals^D,\reals^C)^*
\end{align}
where the notation $X^* = L(X,\reals)$ defines the \emph{dual space}
of the vector space $X$. Thus, an expression
$\vec{W} + \eta\,\nabla L(\vec{W})$ can be evaluated for
$\eta\in\reals$, but $\vec{W} + \eta\, DL(\vec{W})$ makes no sense.

We recall that the notion of gradient depends on the inner product in
the underlying vector space. In our case, it is the space of weight
matrices $\vec{W}$, i.e. $L(\reals^D,\reals^C)$. We assume the most
simple form of the inner product: the Frobenius (also Hilbert-Schmidt) inner product:
\begin{equation}
\label{eqn:inner-product}
\langle \vec{U}, \vec{V} \rangle
= \sum_{j=1}^C\sum_{k=1}^D U_{jk} V_{jk} = \tr{\left(\vec{U}^T\, \vec{V}\right)}.
\end{equation}
We then have the definition of gradient by duality:
\begin{equation}
  DL(\vec{W})\vec{V} = \langle \nabla L(\vec{W}), \vec{V} \rangle.
\end{equation}
for all matrices $\vec{V}$. Then, for $N=1$ we look for
\[ -(\vec{t}-\vec{y})^\trans\,\vec{V}\,\vec{x} = \tr\left\{ \nabla L(\vec{W})^\trans\, \vec{V}\right\} \]
Using the identity $\tr(\vec{A}\,\vec{B}) = \tr(\vec{B}\,\vec{A})$,
and for vectors $\vec{a}$ and $\vec{b}$,
$\vec{a}^\trans \vec{b} = \tr( \vec{b}\,\vec{a}^\trans)$,
we obtain:
\[ -\tr\left\{ \vec{V}\,\vec{x} (\vec{t}-\vec{y})^\trans \right\}
  = \tr\left\{ \vec{V} (\nabla L(\vec{W}))^\trans\right\}
\]
or, more invariantly,
\[ -\left\langle \vec{V}, (\vec{t}-\vec{y})\vec{x}^\trans \right\rangle
  = \langle \vec{V}, \nabla L(\vec{W}) \rangle. \]
Because $\vec{V}$ is arbitrary:
\[ \nabla L(\vec{W}) = -(\vec{t}-\vec{y})\vec{x}^\trans. \]
For arbitrary $N$ we obtain the following gradient formula:
\begin{equation}
  \label{eqn:gradient-formula}
  \nabla L(\vec{W}) = -\sum_{n=1}^N\left(\vec{t}^{(n)}-\vec{y}^{(n)}\right)\left(\vec{x}^{(n)}\right)^\trans
  = -(\vec{T}-\vec{Y})\,\vec{X}^\trans.
\end{equation}
where
\begin{align*}
  \vec{T} &= \begin{bmatrix}
    \vec{t}^{(1)} & \vec{t}^{(2)} & \ldots &\vec{t}^{(N)}
  \end{bmatrix},\\
  \vec{Y} &= \begin{bmatrix}
    \vec{y}^{(1)} & \vec{y}^{(2)} & \ldots &\vec{y}^{(N)}
  \end{bmatrix},\\
  \vec{X} &= \begin{bmatrix}
    \vec{x}^{(1)} & \vec{x}^{(2)} & \ldots &\vec{x}^{(N)}
  \end{bmatrix},
\end{align*}                                           
are matrices containing the elements of the training data in their columns.
Thus, calculating the gradient can be expressed through simple matrix
arithmetic, which leads to very efficient implementations of the gradient
method.

\begin{corollary}[Characterization of critical points]
  A weight matrix $\vec{W}$ is a critical point of $L$ iff
  the sample uncentered covariance matrix of the errors
  \[ \vec{e}^{(n)} = \vec{t}^{(n)} -\vec{y}^{(n)}\]
  and of the input vectors $\vec{x}^{(n)}$ is zero, where
  by definition the sample uncentered covariance matrix is the right-hand side
  of \eqref{eqn:gradient-formula}.
\end{corollary}

\section{The second derivative and the Hessian}
\label{section:hessian}
In this section we consider the second (Fr\'echet) derivative of the loss function
which is closely related to the Hessian.
\subsection{Positivity of the Hessian for $N=1$}
We recall the
expression for the loss function, when $N=1$:
\[ L = Q_{\vec{t}}\circ\boldsymbol\rho \circ P_{\vec{x}}. \]
In view of linearity of the $Q_\vec{t}$ and $P_{\vec{x}}$, we have:
\[ D^2 L(\vec{W}) (\vec{U}, \vec{V})= Q_\vec{t}\,D^2\boldsymbol\rho(\vec{a})(P_\vec{x}\,\vec{U} ,P_\vec{x}\,\vec{V}) \]
where $\vec{a}=P_{\vec{x}}(\vec{W}) = \vec{W}\,\vec{x}$.
We recall that the second derivative $D^2L(\vec{W})$ is a bi-linear, symmetric
function on pairs of vectors $\vec{U}$, $\vec{V}$, which in our case are also
linear transformations in $L(\reals^D,\reals^C)$, i.e. $C\times D$ matrices upon
identifying linear transformations with their matrices with respect to the standard basis.

We also have the formula for $D\boldsymbol\rho$:
\[ D\boldsymbol\rho(\vec{u}) = -\vec{I} + \mathbbm{1}\,\boldsymbol\sigma(\vec{u})^\trans. \]
Differentiating again, we obtain
\begin{align*}
  D^2\boldsymbol\rho(\vec{a})\,(\vec{g},\vec{h})
  = \mathbbm{1}\, (D\boldsymbol\sigma(\vec{a})\,\vec{g})^\trans\,\vec{h}
  = \mathbbm{1}\, \vec{g}^\trans D\boldsymbol\sigma(\vec{a})^\trans\,\vec{h}.
\end{align*}
Therefore,
\[  Q_{\vec{t}} D^2\boldsymbol\rho(\vec{a})\,(\vec{g},\vec{h}) =
  (\vec{t}^\trans \mathbbm{1})
  \vec{g}^\trans D\boldsymbol\sigma(\vec{a})^\trans\,\vec{h}  =
  \vec{g}^\trans D\boldsymbol\sigma(\vec{a})^\trans\,\vec{h}
\]
Hence,
\begin{align*}
  D^2L(\vec{W}) (\vec{U},\vec{V})
  &= (P_{\vec{x}}\vec{U})^\trans\,D\boldsymbol\sigma(\vec{a})^\trans\,(P_{\vec{x}}\vec{V})\\
  &= (\vec{U}\vec{x})^\trans D\boldsymbol\sigma(\vec{a})^\trans (\vec{V}\,\vec{x}) \\
  &= \vec{x}^\trans\,\vec{U}^\trans\, D\boldsymbol\sigma(\vec{a})^\trans\,\vec{V}\,\vec{x}.
\end{align*}

We also find
\begin{align*}
  D\boldsymbol\sigma(\vec{a})
  &= D\exp(-\boldsymbol\rho(\vec{a}))\\
  &= -\diag(\exp(-\boldsymbol\rho(\vec{a}))) D\boldsymbol\rho(\vec{a}) \\
  &= -\diag(\boldsymbol\sigma(\vec{a}))
    \left(-\vec{I} + \mathbbm{1}\,\boldsymbol\sigma(\vec{a})^\trans\right)\\
  &= \diag(\boldsymbol\sigma(\vec{a})) - \diag(\boldsymbol\sigma(\vec{a})) \mathbbm{1}\,\boldsymbol\sigma(\vec{a})^\trans.
\end{align*}
Hence,
\begin{align*}
  D\boldsymbol\sigma(\vec{a}) = 
  \diag(\boldsymbol\sigma(\vec{a}))
  - \boldsymbol\sigma(\vec{a})\,\mathbbm{1}^\trans \, \diag(\boldsymbol\sigma(\vec{a})). 
\end{align*}
Therefore, taking into account that $\diag(\vec{s})\,\mathbbm{1} = \vec{s}$
and $\mathbbm{1}^\trans\,\diag(\vec{s}) = \vec{s}^\trans$, we obtain:
\begin{equation}
  \label{eqn:D-sigma}
  D\boldsymbol\sigma(\vec{a})= \diag(\boldsymbol\sigma(\vec{a}))
  - \boldsymbol\sigma(\vec{a})\,\boldsymbol\sigma(\vec{a})^\trans.
\end{equation}
Clearly, this matrix is symmetric.
\begin{align*}
  D^2L(\vec{W})(\vec{U},\vec{V})
  &= \vec{x}^\trans \vec{U}^\trans \diag(\boldsymbol\sigma(\vec{a})) \vec{V}\,\vec{x}
    -\vec{x}^\trans \vec{U}^\trans \boldsymbol\sigma(\vec{a})\,\boldsymbol\sigma(\vec{a})^\trans\,\vec{V}\,\vec{x}\\
  &=\vec{x}^\trans\,\vec{U}^\trans (\diag(\boldsymbol\sigma(\vec{a})) - \boldsymbol\sigma(\vec{a})\,\boldsymbol\sigma(\vec{a})^\trans) \vec{V}\,\vec{x}\\
  &=\vec{x}^\trans\,\vec{U}^\trans (\diag(\vec{y}) - \vec{y}\,\vec{y}^\trans) \vec{V}\,\vec{x}.    
\end{align*}
This bi-linear form is non-negative definite, as
\[ D^2L(\vec{W})(\vec{U},\vec{U}) = (\vec{U}\,\vec{x})^\trans
  \left(\diag(\vec{y}) - \vec{y}\,\vec{y}^\trans\right) (\vec{U}\,\vec{x}) \ge
  0. \]
Indeed, the matrix $\diag(\vec{y}) - \vec{y}\,\vec{y}^\trans$ is
symmetric and thus has real spectrum. It suffices to show that the
spectrum is non-negative.
\begin{proposition}
  \label{proposition:q-matrix}
  Let $\vec{y}\in\reals^C$ be a vector satisfying
  \begin{enumerate}
  \item $y_i > 0$;
  \item $\sum_{i=1}^C y_i=1$; equivalently, $\mathbbm{1}^\trans\,\vec{y}=1$.
  \end{enumerate}
  Then the eigenvalues of the symmetric matrix
  \[ \vec{Q}=\diag(\vec{y}) - \vec{y}\,\vec{y}^\trans \]
  are all non-negative. Furthermore, the only eigenvector
  with eigenvalue $0$ is $\mathbbm{1}$ up to a scalar factor.
\end{proposition}
\begin{proof}
\noindent\textbf{Method 1:}
Let $\lambda$ be an eigenvalue and $\vec{z}$ be an eigenvector for eigenvalue $\lambda$; thus
\[ \diag(\vec{y})\,\vec{z} - \vec{y}\left(\vec{y}^\trans\,\vec{z}\right) = \lambda\, \vec{z}.\]
Therefore,
\begin{align*}
  y_i\, z_i - y_i\, \langle \vec{y}, \vec{z} \rangle &= \lambda\, z_i. \\
  (y_i-\lambda) z_i &= y_i\,\langle \vec{y}, \vec{z} \rangle
\end{align*}
We know that $0<y_i < 1$. Therefore, unless $0<\lambda<1$, $y_i\neq \lambda$, and
\[ z_i = \frac{y_i}{y_i-\lambda} \langle \vec{y}, \vec{z} \rangle. \]
Since $\vec{z}\neq 0$, we assume WLOG that $\langle \vec{y}, \vec{z} \rangle\neq 0$.
Hence
\[ \langle \vec{y}, \vec{z} \rangle = \sum_{i=1}^C \frac{y_i^2}{y_i-\lambda}\langle \vec{y}, \vec{z} \rangle.\]
and 
\[ \sum_{i=1}^C \frac{y_i^2}{y_i-\lambda} = 1.\]
We need to show that all roots $\lambda$ of this equation are non-negative. Indeed, if $\lambda < 0 $ then
\[ \sum_{i=1}^C \frac{y_i^2}{y_i-\lambda} < \sum_{i=1}^C \frac{y_i^2}{y_i} = \sum_{i=1}^C y_i = 1 \]
which is a contradition. Thus $\lambda \ge 0$.

It remains to see that the eigenvector $\vec{z}$ for the eigenvalue $\lambda=0$ is $\vec{z}=\mathbbm{1}$ up to
a multiplicative constant.
Indeed, if $\lambda = 0$ then for $i=1,2,\ldots,C$:
\[ y_i\, z_i = y_i\, \langle \vec{y}, \vec{z} \rangle. \]
Dividing by $y_i>0$, we obtain:
\[ z_i = \langle \vec{y}, \vec{z} \rangle. \]
Hence $z_i$ is independent of $i$, i.e. proportional to $\mathbbm{1}$.

\noindent\textbf{Method 2:} (Gershgorin Circle Theorem)
The matrix in question has diagonal entry $y_i - y_i^2$ at position $i$, and the
off-diagonal entries in row $i$ are  $y_iy_j$, $j=1,2,\ldots,C$, $j\neq i$.
Hence, for every eigenvalue $\lambda$ there esists $i$ such that:
\[ |\lambda - (y_i-y_i^2) | \leq \sum_{j=1\atop j\neq i}^C y_i y_j. \]
In particular
\[ \lambda \geq (y_i-y_i^2) - \sum_{j=1\atop j\neq i}^C y_i y_j = y_i - y_i\sum_{j=1}^C y_j = y_i-y_i = 0. \]

\noindent\textbf{Method 3:} Let $\vec{M}= \diag(\sqrt{\vec{y}})$. Obviously this is a symmetric diagonal
matrix. Then we have the following factorization:
\[ \diag(\vec{y}) - \vec{y}\,\vec{y}^\trans = \vec{M}^\trans \left(\vec{I} - \vec{u}\,\vec{u}^\trans\right)\vec{M} \]
where $\vec{u} = M^{-1}\,\vec{y} = \sqrt{\vec{y}}$. Clearly, $\sum_{i=1}^Cu_i^2 = \sum_{i=1}^Cy_i = 1$.
Thus $\vec{u}$ is a unit vector. The matrix $\vec{I}-\vec{u}\,\vec{u}^\trans$ is non-negative definite
as it is an orthogonal projection on the hyperplane normal to $\vec{u}$, and therefore its eigenvalues
are $0$ (multiplicity $1$) and $1$ (multiplicity $C-1$).  Moreover, the eigenvector
with eigenvalue $0$ is $\vec{u}$. Hence, the eigenvector $\vec{z}$ with eigenvalue $0$ for
$\diag(\vec{y}) - \vec{y}\,\vec{y}^\trans$ satisfies (up to a scalar multiple):
\[ \vec{M}\,\vec{z} = \vec{u}.\]
and $\vec{z} = \vec{M}^{-1}\,\vec{u} = \mathbbm{1}$.
\end{proof}
We proceed to further investigate the quadratic form
\[ B(\vec{U},\vec{V}) := D^2L(\vec{W})(\vec{U},\vec{V}) = (\vec{U}\,\vec{x})^\trans
  \left(\diag(\vec{y}) - \vec{y}\,\vec{y}^\trans\right) (\vec{V}\,\vec{x})
  \]
\begin{proposition}
  \label{proposition:operator-h}
  Let $\vec{x}\in\reals^D$ and $\vec{y}\in\reals^C$ and let $B$ be the billinear form given by:
  \[ B(\vec{U},\vec{V}) =
    (\vec{U}\,\vec{x})^\trans\left(\diag(\vec{y}) -
      \vec{y}\,\vec{y}^\trans\right) (\vec{V}\,\vec{x})\] There exists
  a unique operator
  $\vec{H}:L(\reals^D,\reals^C)\to L(\reals^D,\reals^C)$ such that for
  every $\vec{U},\vec{V}\in L(\reals^C,\reals^C)$:
  \[ B(\vec{U},\vec{V}) = \langle \vec{H}(\vec{U}), \vec{V} \rangle
    =\langle \vec{U}, \vec{H}(\vec{V}) \rangle. \]
  Thus, by definition, $\vec{H}$ is a symmetric operator.
  Moreover, $\vec{H}$ is given explicitly by the formula:
  \[ \vec{H}(\vec{U}) = \vec{Q}\,\vec{U}\,\vec{P} \]
  where $\vec{Q}=\diag(\vec{y})-\vec{y}\,\vec{y}^\trans$ and $\vec{P}=\vec{x}\,\vec{x}^\trans$.
\end{proposition}
\begin{proof}
By direct computation:
\begin{align*}
  B(\vec{U},\vec{V})
  &= (\vec{U}\,\vec{x})^\trans\,\vec{Q}\, (\vec{V}\,\vec{x})
  =\tr \left( \left(\vec{U}\,\vec{x}\right)^\trans\,\vec{Q} (\vec{V}\,\vec{x}) \right)
  =\tr \left( \vec{x}^\trans\,\vec{U}^\trans\,\vec{Q} \vec{V}\,\vec{x} \right)\\
  &=\tr \left( \vec{x}\,\vec{x}^\trans\vec{U}^\trans\,\vec{Q}\, \vec{V}\,\right)
  = \tr \left( \vec{P}\,\vec{U}^\trans\,\vec{Q}\,\vec{V}\,\right)  
  = \tr \left( \left(\vec{Q}\,\vec{U}\,\vec{P}\right)^\trans\,\vec{V}\,\right)  
  =\langle\vec{H}(\vec{U}),\vec{V}\,\rangle.
\end{align*}
We exploited the invariance of trace under cyclic permutation of the matrix factors. The
proof of $B(\vec{U},\vec{V})=\langle \vec{U}, \vec{H}(\vec{V}) \rangle$ is left to the reader.
\end{proof}
Method 3 in the proof of the Proposition~\ref{proposition:q-matrix}
leads to an interesting refinement:
\begin{proposition}
  \label{proposition:q-matrix-refinement}
  Let $\vec{y}\in\reals^C$ be a vector satisfying
  \begin{enumerate}
  \item $y_i > 0$;
  \item $\sum_{i=1}^C y_i=1$.
  \end{enumerate}
  Let $\vec{x}\in\reals^D$ be an arbitrary vector. Let
  $ \vec{Q}=\diag(\vec{y}) - \vec{y}\,\vec{y}^\trans$ be a symmetric
  matrix (non-negative definite by
  Proposition~\ref{proposition:q-matrix}).  Let
  $ \vec{P} = \vec{x}\,\vec{x}^\trans$.  Then
  \[\left\langle \vec{Q}\,\vec{U}\,\vec{P},\vec{U} \right\rangle
    = \| \vec{R}\,\vec{M}\,\vec{U}\,\vec{P}\|^2 \]
  where $\vec{R} = \vec{I} - \vec{u}\,\vec{u}^\trans$, where $\vec{u}=\sqrt{\vec{y}}$ is a unit vector,
  and where $\vec{M} = \diag\left(\sqrt{\vec{y}}\right) = \diag\left(\vec{u}\right)$.
\end{proposition}
\begin{proof}
  We have $\vec{Q} = \vec{M}^\trans\,\vec{R}\,\vec{M}$. Therefore
  \begin{align*}
    \langle \vec{Q}\,\vec{U}\,\vec{P},\vec{U}\rangle
    &=\tr\left(\left(\vec{Q}\,\vec{U}\,\vec{P}\right)^\trans\vec{U}\right)
      =\tr\left( \vec{P}\,\vec{U}^\trans\,\vec{Q}\,\vec{U}\right)
      =\tr\left( \vec{P}\,\vec{U}^\trans\,\vec{M}^\trans\vec{R}\vec{M}\,\vec{U}\right)\\
    &=\tr\left( \vec{P}^2\,\vec{U}^\trans\,\vec{M}^\trans\,\vec{R}\,\vec{M}\,\vec{U}\right)
      =\tr\left( \vec{P}\,\vec{U}^\trans\,\vec{M}^\trans\,\vec{R}\,\vec{M}\,\vec{U}\,\vec{P}\right)\\
    &=\tr\left( \left(\vec{R}\,\vec{M}\,\vec{U}\,\vec{P}\right)^\trans\vec{R}\,\vec{M}\,\vec{U}\,\vec{P}\right)=\| \vec{R}\,\vec{M}\,\vec{U}\,\vec{P}\|^2 .
  \end{align*}
\end{proof}

\begin{corollary}
  \label{corollary:summary-of-definiteness}
  Under the assumptions of Proposition~\ref{proposition:q-matrix-refinement}, let
  $\vec{H}:L(\reals^D,\reals^C)\to L(\reals^D,\reals^C)$ be a linear operator given by:
  $\vec{H}(\vec{U}) = \vec{Q}\,\vec{U}\,\vec{P}$.
  Then $\vec{H}$ is a symmetric operator with respect to the Frobenius inner product
  and the associated quadratic form $B$ given by
  $B(\vec{U},\vec{V}) = \left\langle \vec{H}(\vec{U}), \vec{V}\right\rangle$
  is non-negative definite. Moreover,
  \[ \kernel{\vec{H}} = \{\vec{U}\,:\, \vec{U}\,\vec{x}\in\vspan{\{\mathbbm{1}\}}\}.\]
\end{corollary}
\begin{proof}
  By Proposition~\ref{proposition:q-matrix-refinement},
  $\vec{H}(\vec{U})=0$ implies
  $ \vec{R}\,\vec{M}\,\vec{U}\,\vec{P} = 0 $.  We multiply this
  equation by $\vec{x}$ on the right and get:
  $ \vec{R}\,\vec{M}\,\vec{U}\,\vec{P}\,\vec{x} = 0 $.  We have
  $\vec{P}\,\vec{x}=\|\vec{x}\|^2\vec{x}$, therefore
  $\vec{R}\,\vec{M}\,\vec{U}\,\vec{x} = 0 $.  As $\vec{R}$ is a
  projection on $\vec{u}^\perp$, this implies the existence of
  $c\in\reals$ such that $ \vec{M}\,\vec{U}\,\vec{x} = c\,\vec{u}$.
  In turn,
  $\vec{U}\,\vec{x} = c\,\vec{M}^{-1}\,\vec{u} = c\,\mathbbm{1}$, as
  $\vec{M}^{-1}\vec{u}=\mathbbm{1}$. This proves ``$\subseteq$''. The
  inclusion ``$\supseteq$'' is straightforward.
\end{proof}

\subsection{Positivity of the full Hessian}
Let us consider arbitrary $N$ (the size of the training set). We then have
\[  D^2L(\vec{W})(\vec{U},\vec{U}) = \sum_{n=1}^N\left(\vec{U}\,\vec{x}^{(n)}\right)^\trans
  \left(\diag\left(\vec{y}^{(n)}\right) - \vec{y}^{(n)}\,\left(\vec{y}^{(n)}\right)^\trans\right) \left(\vec{U}\,\vec{x}^{(n)}\right) \ge 0 \]
because every term is non-negative.

By Corollary~\ref{corollary:summary-of-definiteness}, the only way to
get a value of $0$ is for every summand to be $0$, i.e. to have for
$n=1,2,\ldots,N$:
\[ \vec{U}\,\vec{x}^{(n)} = c_n \,\mathbbm{1} \]
for some scalars $c_n$, $n=1,2,\ldots,N$. These equations can also be written as a single
matrix identity:
\begin{equation}
  \label{eqn:degeneracy-condition}
  \vec{U}\,\vec{X} = \mathbbm{1}\,\vec{c}^\trans
\end{equation}
where $\vec{c}=(c_1,c_2,\ldots,c_N)$.

The degeneracy condition \eqref{eqn:degeneracy-condition} is rather
hard to fulfill for a randomly chosen training set. If the coordinates
of $\vec{x}^{(n)}$ (the \emph{features}, in the language of machine
learning) are linearly independent as random variables then $\vec{X}$
has maximum rank $D$ with probability $1$ if the variables are
continuous, or with probability asymptotically converging to $1$ if
the variables are discrete (or mixed), as $N\to\infty$. Let
$\tilde{\vec{X}}$ be a $D\times D$ non-singular submatrix of $\vec{X}$
(a subset of columns of $\vec{X}$). Then clearly
\[ \vec{U}\,\tilde{\vec{X}} = \mathbbm{1}\,\tilde{\vec{c}}^\trans.\]
where $\tilde{\vec{c}}$ is the corresponding subset of entries of $\vec{c}$.
Hence,
\[ \vec{U} = \mathbbm{1}\,\tilde{\vec{c}}^\trans \,\tilde{\vec{X}}^{-1}
  = \mathbbm{1}\,\tilde{\tilde{\vec{c}}}^\trans.
\]
where $\tilde{\tilde{\vec{c}}} = \left(\tilde{\vec{X}}^{-1}\right)^\trans\,\tilde{\vec{c}}$.
Thus, all columns of $\vec{U}$ must be multiples of $\mathbbm{1}$.  We
are already familiar with this condition. If we restrict $L$ to only
matrices $\vec{W}$ for which the mean of every column is $0$ then the
second derivative $D^2L(\vec{W})$ is strictly-positive definite, thus
ensuring convergence of optimization methods, such as gradient
descent. More precisely, the expression $D^2L(\vec{W})(\vec{U},\vec{U})$
is strictly positive for every matrix $\vec{U}\neq 0$ whose column means are $0$.
Indeed, the only matrix $\vec{U}$ satisfying both $\mathbbm{1}^\trans\,\vec{U} = 0$
and $\vec{U} = \mathbbm{1}\,\vec{c}^\trans$ is $\vec{U}=0$:
\[\mathbbm{1}^\trans\,\vec{U} = \left(\mathbbm{1}^\trans\,\mathbbm{1}\right)\vec{c}^\trans = C\,\vec{c}^\trans = 0,\]
so $\vec{c}=0$, and thus $\vec{U}=0$.
\begin{corollary}[On convergence and regularization]
  The loss function $L$ for the multi-class logistic regression is
  strictly convex on the set of weight matrices $\vec{W}$ with column
  mean zero, as long as the training set spans the vector space
  $\reals^D$, or, equivalently, when $\rank\vec{X} = D$.  Therefore,
  it is possible to find the optimal weight matrix $\vec{W}$ without
  using a regularizer. If this condition is satisfied,
  there exists at most one global minimum $\vec{W}$ of $L$  with column means $0$.
  Thus, the global minimum is unique, \textbf{if it exists}.

  Other minima which do not have column mean $0$ are obtained
  by shifting each column of $\vec{W}$ by a constant scalar (possibly
  different for each column).
\end{corollary}

\section{A sample training algorithm}
\label{section:sample-training-algorithm}
Algorithm~\ref{alg:training-algorithm} contains the basic training loop
for the logistic regression network we described in previous sections.
Many modifications are possible. For instance, one can implement
variable training rate, utilizing, for example, the Barzilai-Borwein
update rule \cite{barzilai-borwein}. Also, we can break out of the loop if no progress is made.
We could evaluate the loss function and see if it decays, as a test
of progress, etc.

\begin{algorithm}
  \caption{The algorithm implements training of the logistic
    regression network by gradient descent method. The
    \textsc{softmax} function is simply $\boldsymbol\sigma$ defined by
    equation~\eqref{eqn:sigma}, applied
    columnwise. \label{alg:training-algorithm}}
  \begin{algorithmic}
    \Require
    \Statex $\vec{X}$ is an $D\times N$ matrix of rank $D$, containing the $N$ training
    vectors as columns;
    \Statex $\vec{W}$ is a $C\times D$ weight matrix initialized at random;
    \Statex $\eta\in\reals$ is the learning rate;
    \Statex $NumEpochs$ is the number of epochs;
    \Ensure
    \Statex $\vec{W}$ approximates the optimal weights minimizing
    the cross-entropy loss function.
    \For{$epoch = 1,2,\ldots,NumEpochs$}
    \State{$\vec{A} \gets \vec{W}\cdot\vec{X}$} \Comment{Compute activations.}
    \State{$\vec{Y} \gets \Call{softmax}{\vec{A}}$} \Comment{Compute softmax activity.}
    \State{$\vec{E} \gets \vec{T} - \vec{Y}$}     \Comment{Compute errrors.}
    \State{$\boldsymbol\nabla L \gets -\vec{E} \cdot \vec{X}^\trans$} \Comment{Find the gradient.}
    \State{$\vec{W} \gets \vec{W} - \eta \cdot \boldsymbol\nabla L$ } \Comment{Update weights.}
    \EndFor
  \end{algorithmic}
\end{algorithm}
It should be noted tat
\[ \mathbbm{1}^\trans \,\vec{T}=\mathbbm{1}^\trans\,\vec{Y} = 1\]
and therefore
\[ \mathbbm{1}^\trans \left(\vec{T}-\vec{Y}\right) = 0\]
Hence, in view of $\boldsymbol\nabla L(\vec{W}) = (\vec{T}-\vec{Y})\,\vec{X}^\trans$,
\[ \mathbbm{1}^\trans \,\boldsymbol\nabla L(\vec{W}) = 0 \]
and the condition
\[ \mathbbm{1}^\trans \,\vec{W} = 0 \]
is maintained during the iteration process, at least if round-off error
is ignored. However, we may want to subtract the mean from each column of $\vec{W}$
every so often to prevent round-off error acting as diffusion in the weight space, which
clearly prevents convergence.

\subsection{A reference implementation}
An reference implementation in MATLAB is available on-line
\cite{rychlik-multiclass-logistic-regression}, in the context of a
bigger optical character recognition project. Barzilai-Borwein update
rule \cite{barzilai-borwein} is used to control the learning rate. The
implementation uses the popular MNIST database of handrwitten digits
\cite{mnist-database} as test data.

\section{A rigorous bound on the rate of convergence of Gradient Descent}
\label{section:convergence-rate}
We proceed to analyze the existence and uniqueness of the minimum of the
loss function. We use a rather conventional set of tools to achieve this,
with significant success.
\subsection{Stability and convergence rate of Gradient Descent}
\label{section:gradient-descent}
The training algorithm in section~\ref{section:sample-training-algorithm}
with fixed learning rate $\eta$ is the Euler method for solving
the gradient differential equation
\begin{equation}
  \label{eqn:gradient-flow}
  \dot{\vec{W}} = - \boldsymbol\nabla L(\vec{W}).
\end{equation}
We define the subspace
\[Z = \left\{ \vec{W}\in L(\reals^D,\reals^C) \,:\, \mathbbm{1}^\trans\,\vec{W}=0\right\}\]
of whose $\vec{W}$ columns sum up to $0$. This subspace 
is invariant under the gradient flow. The stability theory for
ordinary differential equations tells us that the asymptotic
convergence rate for this ODE at the global minimum $\hat{\vec{W}}$ of
$L$ is the same as for the linearized system:
\begin{equation}
  \label{eqn:linearized-gradient-flow}
  \dot{\vec{U}} = - \left(D\left(\boldsymbol\nabla{L}\right)\right)(\hat{\vec{W}})\,\vec{U}
\end{equation}
where $\vec{U}=\vec{W}-\hat{\vec{W}}$.
In turn, the linear theory of ODE says that the solution
\[ \|\vec{U}(t)\| \approx e^{-\lambda_{min} t} \|\vec{U}(0)\| \]
where $\lambda_{min}$ is the smallest eigenvalue of the operator
\[ \hat{\vec{H}} = \left(D\left(\boldsymbol\nabla{L}\right)\right)(\hat{\vec{W}}) \]
which, under the usual identification of linear operators with
matrices, is the Hessian of $L$ at $\hat{\vec{W}}$. This motivates our
desire to estimate $\lambda_{min}$. We note that, since operator
$\hat{\vec{H}}$ is symmetric, its spectrum is real.
It will also be useful to consider the Hessian evaluated at an arbitrary point:
\[ \vec{H}(\vec{W})=\left(D\left(\boldsymbol\nabla{L}\right)\right)(\vec{W}). \]
We also note an alternative definition of $\vec{H}(\vec{W})$ as the only operator
such that for every $\vec{U},\vec{V}\in L(\reals^D,\reals^C)$:
\begin{equation}
  \label{eqn:H-and-second-derivative}
  D^2L(\vec{W})(\vec{U},\vec{V}) = \langle (\vec{H}(\vec{W}))(\vec{U}),\vec{V}\rangle.
\end{equation}
The equivalence of both definitions follows by differentiating the
definition of the gradient 
\[ DL(\vec{W})(\vec{V}) = \langle \boldsymbol\nabla L(\vec{W}), \vec{V} \rangle \] 
in the direction of $\vec{U}$. We obtain
\begin{equation}
  \label{eqn:H-and-second-derivative-alt}
  D^2L(\vec{W})(\vec{U},\vec{V}) = \langle \left(D(\boldsymbol\nabla L)(\vec{W})\right)(\vec{U}), \vec{V} \rangle.
\end{equation}
By comparison of equations~\eqref{eqn:H-and-second-derivative} and \eqref{eqn:H-and-second-derivative-alt}
we obtain the equivalence of both definitions of $\vec{H}(\vec{W})$.

In most places we will consider $\vec{W}$ fixed and omit the argument
$\vec{W}$ from $\vec{H}(\vec{W})$. We note that the operator $\vec{H}$ is symmetric,
i.e.
\[\langle\vec{H}(\vec{U}),\vec{V}\rangle = \langle\vec{U},\vec{H}(\vec{V})\rangle\]
where the inner product $\langle\cdot,\cdot\rangle$ is the trace inner product given by~\eqref{eqn:inner-product}.
It is worth noting that $\vec{H}$ is \textbf{not a matrix}. As $L:L(\reals^D,\reals^C)\to\reals$,
we have \[\boldsymbol\nabla{L}:L(\reals^D,\reals^C)\to L(\reals^D,\reals^C).\] Therefore
\[D(\boldsymbol\nabla{L}):L(\reals^D,\reals^C)\to L(L(\reals^D,\reals^C),L(\reals^D,\reals^C)).\]
Hence
\[ \vec{H}(\vec{W})=D(\boldsymbol\nabla{L})(\vec{W})\in L(L(\reals^D,\reals^C),L(\reals^D,\reals^C)).\]
The actual formula for the sequence of Euler approximations to the solution of equation is
\[ \vec{W}_{n} = \vec{W}_{n-1} - \eta\,\boldsymbol\nabla L(\vec{W}_{n-1}). \]
and is also Gradient Descent method of machine learning with fixed learning rate.
We could introduce the map
\[ \boldsymbol\Phi(\vec{W}) = \vec{W} - \eta\,\boldsymbol\nabla L(\vec{W}) \]
and prove that this map is a uniform contraction. We find the 
\[ D\boldsymbol\Phi(\vec{W}) = \vec{I} - \eta\,D(\boldsymbol\nabla L(\vec{W})) =
  \vec{I}-\eta\, \vec{H}.\] We have
$\sigma(\vec{H})\subseteq[\lambda_{min},\lambda_{max}]$ where
$\lambda_{min}$ and $\lambda_{max}$ are the extreme (positive, by the
results of section~\ref{section:hessian}) eigenvalues of
$\vec{H}$. Therefore,
\[ \sigma(\vec{I}-\eta\,\vec{H}) \subseteq
  [1-\eta\,\lambda_{max},1-\eta\,\lambda_{min}]. \] It is obvious that
for sufficiently small $\lambda>0$ the spectrum is within $(-1,1)$.
Furthermore, any uniform bound (with respect to $\vec{W}$) for $\lambda_{max}$ from
above, and a uniform positive bound for $\lambda_{min}$ from below
proves that $\boldsymbol\Phi$ is a globally defined contraction.

We would like to have a more quantitative statement, and, in
particular, we would like to determine how to pick $\eta$ for optimum
convergence rate.

In order for $D\boldsymbol\Phi(\vec{W})$ to be a contraction, this last interval must be within
the interval $[-\theta,\theta]$, where $\theta\in(0,1)$. Therefore,
\begin{align*}
  -\theta & \le 1 - \eta\,\lambda_{max},\\
  \theta & \ge 1 - \eta\,\lambda_{min}.
\end{align*}
Therefore, given $\theta\in(0,1)$, we must have:
\[ \frac{1-\theta}{\lambda_{min}} \leq \eta \leq \frac{1+\theta}{\lambda_{max}} \]
In particular, the necessary condition to have a contraction is:
\[ \frac{1-\theta}{\lambda_{min}} \leq \frac{1+\theta}{\lambda_{max}} \]
or
\[ K \le \frac{1+\theta}{1-\theta}. \]
where
\[ K = \frac{\lambda_{max}}{\lambda_{min}} \] is also the
\textbf{condition number} of the Hessian. We recall the definition.
\begin{definition}[Condition number]
  Let $\vec{A}:X\to Y$ be a bounded linear operator between two normed spaces $X$ and $Y$
  and let $\vec{A}^{-1}:Y\to X$ exist and be bounded. Then the condition number is
  \[ K(\vec{A}) = \| \vec{A} \| \|\vec{A}^{-1}\| \]
  where the norm is the operator norm respective of the norms on $X$ and $Y$.
\end{definition}
Thus, let $K=K(\vec{H})=\|\vec{H}\|\|\vec{H}^{-1}\|$. For symmetric operators this is
exactly the same as $\lambda_{max}/\lambda_{min}$. Clearly,
$K\in[1,\infty)$. Therefore, the relationship between $K$ and $\theta$
defining the optimum value of $\theta$ is expressed by these
equations:
\[ K =\frac{1+\theta}{1-\theta}, \qquad \theta=\frac{K-1}{K+1}. \]
In particular, by the Mean Value Theorem the mapping $\boldsymbol\Phi$ is a
\textbf{weak contraction} (Lipschitz constant $\leq 1$) and any
uniform estimate of $K$ (with respect to $\vec{W}$) yields a uniform
estimate on the contraction rate. Moreover, the asymptotic rate is the
value of $K$ at $\vec{W}=\hat{\vec{W}}$.

Upon a quick examination it can be seen that $\boldsymbol\Phi$ is not a
strong contraction. This inhibits our ability to obtain an explicit
bound on $\hat{\vec{W}}$ in a straightforward fashion, by applying the
Banach contraction principle or similar tools. However, the asymptotic
rate estimate can be found if the fixed point is known (or approximated
numerically). We embark on providing such an explicit estimate. In
the current paper we complete the task for two classes ($C=2$) but
in a forthcoming paper we will address the general case, i.e. arbitrary $C$.

\subsection{An abstract class of quadratic forms}
It will be beneficial to develop an algebraic theory which allows to
estimate the spectrum of quadratic forms similar to
$D^2L(\vec{W}) = D(\boldsymbol\nabla L)(\vec{W})$.  We will briefly
formulate the abstract framework for this.

Let $\vec{X}\in L(\reals^N,\reals^D)$,
$\vec{Y}\in L(\reals^N,\reals^C)$ be two matrices and let us assume
that all entries of $\vec{Y}$ are positive and $\mathbbm{1}^\trans\,\vec{Y}=\mathbbm{1}^\trans$.  Let
$\vec{x}^{(n)}\in\reals^D$ be the columns of $\vec{X}$ and
$\vec{y}^{(n)}\in\reals^C$ be the columns of $\vec{Y}$
($n=1,2,\ldots,N$). Let 
\[ \vec{Q}^{(n)} = \diag\left(\vec{y}^{(n)}\right) - \vec{y}^{(n)}\,\left(\vec{y}^{(n)}\right)^\trans.\]
and let $B:L(\reals^D,\reals^C)\times L(\reals^D,\reals^C)\to\reals$ be a symmetric billinear
form defined by:
\[ B(\vec{U},\vec{V}) = \sum_{n=1}^N \left(\vec{U}\,\vec{x}^{(n)}\right)^\trans \vec{Q}^{(n)}\left(\vec{V}\,\vec{x}^{(n)}\right). \]
Let $\vec{H}:L(\reals^D,\reals^C)\to L(\reals^D,\reals^C)$ be a linear operator defined by:
\[ B(\vec{U},\vec{V}) = \langle \vec{H}(\vec{U}),\vec{V}\rangle. \]
By Proposition~\ref{proposition:operator-h}, the operator $\vec{H}$ admits an explicit expression
\begin{equation}
  \label{eqn:H-operator}
  \vec{H}(\vec{U}) = \sum_{n=1}^N\vec{H}^{(n)}(\vec{U}) = \sum_{n=1}^N \vec{Q}^{(n)}\,\vec{U}\, \vec{P}^{(n)}
\end{equation}
where
\begin{align}
  \label{eqn:operator-p}
  \vec{P}^{(n)} &= \vec{x}^{(n)}\,\left(\vec{x}^{(n)}\right)^\trans,\\
  \label{eqn:operator-q}
  \vec{Q}^{(n)} &= \diag\left(\vec{y}^{(n)}\right) - \vec{y}^{(n)}\,\left(\vec{y}^{(n)}\right)^\trans,\\
  \label{eqn:operator-h-n}
  \quad\forall\,\vec{U}\in L(\reals^D,\reals^C)\;&:\;\vec{H}^{(n)}(\vec{U}) = \vec{Q}^{(n)}\,\vec{U}\,\vec{P}^{(n)}.
\end{align}

Let, as before, 
\[Z = \left\{ \vec{W}\in L(\reals^D,\reals^C) \,:\, \mathbbm{1}^\trans\,\vec{W}=0\right\}.\]

\subsection{The determinant lemma}
The first indication that it is possible to estimate the convergence rate
of fixed point iteration (Gradient Descent with fixed learning rate)
comes from a calculation of the determinant.
\begin{lemma}
  Let $\vec{H}_Z=\vec{H}|Z$ be the operator restricted to the
  invariant subspace $Z$.  If additionally $N=D$ then
  \[ \det(\vec{H}_Z) = \alpha_{D,C}\,\det(\vec{X})^{C}\prod_{n=1}^N\prod_{j=1}^C y_j^{(n)} \]
  where $\alpha_{D,C}$ is a constant coefficient depending on $D$ and $C$ only.
\end{lemma}
We will carry out the calculation for the case of $C=2$ classes.  This
suffices for estimating the minimum eigenvalue and the condition
number of the Hessian for two classes. As indicated, the analysis
of the general $C$ will be carried out in another paper, and in particular
the general coefficient $\alpha_{D,C}$ shall be calculated.

\subsection{The case of two classes}
It will be insightful to first consider the case of $C=2$ (two
classes, i.e. the classical logistic regression model).
We note that in this case the conditon $\mathbbm{1}^\trans\,\vec{U}$
imposes the following structure on the matrix $\vec{U}\in L(\reals^D,\reals^2)$:
\begin{equation}
  \vec{U} = \begin{bmatrix}
    u_{11} & u_{12} & \ldots & u_{1D} \\
    -u_{11} & -u_{12} & \ldots & -u_{1D} \\
  \end{bmatrix} =
  \begin{bmatrix}
    1\\-1
  \end{bmatrix}\,
  \begin{bmatrix}
    u_{11} & u_{12} & \ldots & u_{1D}
  \end{bmatrix}
  =\begin{bmatrix}
    1\\-1
  \end{bmatrix}\,\vec{u}^\trans.
\end{equation}
Thus, using the notations
\begin{align*}
  \boldsymbol\xi &= \frac{1}{\sqrt{2}}\,(1,-1),\\
  \vec{u} &= (u_{11},u_{12},\ldots,u_{1D})
\end{align*}
we obtain:
\[ Z= \{ \boldsymbol\xi\,\vec{u}^\trans\,:\, \vec{u}\in\reals^D \}. \]
Moreover, the map $\vec{u}\mapsto\boldsymbol\xi\,\vec{u}^\trans$ is an isometry,
and thus does not change eigenvalues, condition numbers, etc.
We then find that with $\vec{U}=\boldsymbol\xi\,\vec{u}^\trans$ and
$\vec{V}=\boldsymbol\xi\,\vec{v}^\trans$:
\begin{align*}
  b(\vec{u},\vec{v})&=B(\vec{U},\vec{V})
  = \sum_{n=1}^N \left(\boldsymbol\xi\,\vec{u}^\trans\,\vec{x}^{(n)}\right)^\trans\,
    \vec{Q}^{(n)}\,\left(\boldsymbol\xi\,\vec{v}^\trans\,\vec{x}^{(n)}\right) \\
  &= \sum_{n=1}^N \left(\vec{x}^{(n)}\right)^\trans\,\vec{u}\,\boldsymbol\xi^\trans\,\vec{Q}^{(n)}
    \boldsymbol\xi\,\vec{v}^\trans\,\vec{x}^{(n)} \\
\end{align*}  
First, let us prove a lemma which is key in the calculations.
\begin{lemma}
  \label{lemma:q-lemma}
  Let 
  $\vec{Q}=\diag(\vec{y})-\vec{y}\,\vec{y}^\trans$, where
  $\vec{y}=(y_1,y_2)$, and $y_1+y_2=1$. Let $\alpha = 2\,y_1\,y_2$.
  Then
  \begin{enumerate}
  \item $\boldsymbol\xi=(1,-1)/\sqrt{2}$ is an eigenvector of $\vec{Q}$ with eigenvalue $\alpha$;
  \item $\boldsymbol\xi^\trans\,\vec{Q}\,\boldsymbol\xi = \alpha$.
  \end{enumerate}
\end{lemma}
\begin{proof}
  By direct calculation:
  \begin{align*}
    \vec{Q}\,\boldsymbol\xi
    &=\begin{bmatrix}
      y_1-y_1^2 & -y_1\,\,y_2\\
      -y_1\,y_2  & y_2-y_2^2
    \end{bmatrix}
                   \frac{1}{\sqrt{2}}
                   \begin{bmatrix}
                     1\\-1
                   \end{bmatrix}
    = \frac{1}{\sqrt{2}}
    \begin{bmatrix}
      y_1-y_1^2 +y_1y_2\\
      -y_1\,\,y_2-y_2+y_2^2
    \end{bmatrix}\\
    &= \frac{1}{\sqrt{2}}
      \begin{bmatrix}
        y_1(1-y_1 +y_2)\\
        y_2(-y_1-y_2+1)
      \end{bmatrix}
    = \frac{1}{\sqrt{2}}
    \begin{bmatrix}
      2\,y_1\,y_2\\
      -2\,y_1\,y_2
    \end{bmatrix}
    = 2\,y_1\,y_2\,\boldsymbol\xi;
  \end{align*}
  (On second thought, this is not surprising, as
  $\boldsymbol\xi$ is orthogonal to $(1,1)$, which, as we know, is an
  eigenvector, and $\vec{Q}$ is symmetric.)
\end{proof}
We apply Lemma~\ref{lemma:q-lemma} to $\vec{Q}^{(n)}=\diag\left(\vec{y^{(n)}}\right)-\vec{y}^{(n)}\,\left(\vec{y}^{(n)}\right)^\trans$,
$\vec{y}^{(n)} = (y_1^{(n)},y_2^{(n)})$, for $n=1,2,\ldots,N$. Using the notation
\[\alpha_n:=\boldsymbol\xi^\trans\vec{Q}^{(n)}\boldsymbol\xi = 2\,y_1^{(n)}\,y_2^{(n)}\]
we find that
\begin{align*}
  b(\vec{u},\vec{v})
  &=\sum_{n=1}^N \alpha_n\,\left(\vec{x}^{(n)}\right)^\trans\,\vec{u}\,\vec{v}^\trans\,\vec{x}^{(n)}
  =\sum_{n=1}^N \alpha_n\,\left(\left(\vec{x}^{(n)}\right)^\trans\,\vec{u}\right)\,\left(\vec{v}^\trans\,\vec{x}^{(n)}\right)\\
  &=\sum_{n=1}^N \alpha_n\,\left(\vec{v}^\trans\,\vec{x}^{(n)}\right)\,\left(\left(\vec{x}^{(n)}\right)^\trans\,\vec{u}\right)
  =
    \vec{v}^\trans\,
    \left(\sum_{n=1}^N \alpha_n\,\vec{x}^{(n)}\left(\vec{x}^{(n)}\right)^\trans\right)
    \,\vec{u}\\
  &=\vec{v}^\trans\,\left(\vec{X}\,\diag(\boldsymbol\alpha)\,\vec{X}^\trans\right)\,\vec{u}.
\end{align*}
This implies that
\[ \vec{H}_Z = \vec{X}\,\diag(\boldsymbol\alpha)\,\vec{X}^\trans\]
where, as usual, we identify linear operators on $\reals^D$, with their matrices.
In particular, when $D=N$,
\begin{align*}
  \det(\vec{H}_Z) &= \det(\vec{X})\,\det(\diag(\boldsymbol\alpha))\,\det(\vec{X}^\trans)
  = \left(\prod_{n=1}^N\alpha_n\right)\,\det(\vec{X})^2\\
  &= 2^N\,\left(\prod_{n=1}^n y_1^{(n)}\,y_2^{(n)}\right)\,\det(\vec{X})^2.
\end{align*}
Due to the submultiplicative property of the condition number (for
every pair of invertible matrices $\vec{A}$, $\vec{B}$ of the same
size, $K(\vec{A}\,\vec{B})\leq K(\vec{A})\,K(\vec{B})$), we obtain:
\[
  K(\vec{H}_Z) \leq K(\vec{X})^2\,K(\diag(\boldsymbol\alpha)) =
  K(\vec{X})^2\,\dfrac{\displaystyle\max_{1\leq n\leq N} y_1^{(n)}\,y_2^{(n)}}{\displaystyle\min_{1\leq n\leq N} y_1^{(n)}\,y_2^{(n)}}.
\]

Alternatively, we can compute
$\vec{H}\left(\boldsymbol\xi\vec{u}^\trans\right)$ directly,
using~\eqref{eqn:H-operator}. For comparison, we include the
calculation. 
Applying
equation~\eqref{eqn:H-operator} we obtain
\begin{align*}
  \vec{H}\left(\boldsymbol\xi\,\vec{u}^\trans\right)
  &= \sum_{n=1}^N \vec{Q}^{(n)}\,\boldsymbol\xi\,\vec{u}^\trans\, \vec{P}^{(n)}
    = \sum_{n=1}^N\alpha_n\,\boldsymbol\xi\,\vec{u}^\trans\, \vec{P}^{(n)}
    = \sum_{n=1}^N\alpha_n\,\boldsymbol\xi\,\vec{u}^\trans\,\vec{x}^{(n)}\,\left(\vec{x}^{(n)}\right)^\trans\\  
  &= \boldsymbol\xi\,\vec{u}^\trans\left(\sum_{n=1}^N\alpha_n\,\vec{x}^{(n)}\,\left(\vec{x}^{(n)}\right)^\trans\right)  
    = \boldsymbol\xi\,\left(\left(\sum_{n=1}^N\alpha_n\,\vec{x}^{(n)}\,\left(\vec{x}^{(n)}\right)^\trans\right)\,\vec{u}\right)^\trans\\
  &= \boldsymbol\xi\, \left( \left(\vec{X}\,\diag(\boldsymbol\alpha)\,\vec{X}^\trans\right)\,\vec{u}\right)^\trans.
\end{align*}
This equation shows that $\vec{H}_Z$ acts on $\vec{u}$ by left-multiplying it by the matrix
$\vec{M}:=\vec{X}\,\diag(\boldsymbol\alpha)\,\vec{X}^\trans$. We can also represent this equation
by a commuting diagram:
\[
  \begin{tikzcd} [%
    sep=huge,
    ,every label/.append style={font=\normalsize}
    ]
    \reals^D \arrow{d}{\tilde{\vec{H}}_Z} \arrow{r}{\vec{J}}    & Z \arrow{r}{\subseteq} \arrow[swap]{d}{\vec{H}_Z} & L(\reals^D,\reals^2)\arrow[swap]{d}{\vec{H}}  \\%
    \reals^D \arrow{r}{\vec{J}} & Z  \arrow{r}{\subseteq} & L(\reals^D,\reals^2)
  \end{tikzcd}
\]
where
\begin{align*}
  \tilde{\vec{H}}_Z&: \vec{u}\mapsto\vec{M}\,\vec{u},\\
  \vec{J}&: \vec{u}\mapsto\boldsymbol\xi\,\vec{u}^\trans.
\end{align*}
Therefore, since $\vec{J}$ is an isometry, the operators $\vec{H}_Z$
and $\tilde{\vec{H}}_Z$ have the same eigenvalues, and condition
number.

\begin{corollary}
  The condition number of the Hessian of the loss function $L$ at
  the point $\vec{W}\in Z$ is subject to the following inequality:
  \[
    K(\vec{H}_Z) \leq K(\vec{X})^2\,\dfrac{\displaystyle\max_{1\leq
        n\leq N} y_1^{(n)}\,y_2^{(n)}}{\displaystyle\min_{1\leq n\leq N}
      y_1^{(n)}\,y_2^{(n)}}
  \]
  where the activations $y_1^{(n)}$ and $y_2^{(n)}$ are calculated
  at $\vec{W}$. If $\vec{W}$ is known, or effective bounds for $\vec{W}$
  exist, the rate of convergence of the gradient descent admits
  an effective bound as well.
\end{corollary}

\appendix
\section{Spectral properties of the Hessian-related matrix}
In this section we obtain a stronger version of Proposition~\ref{proposition:q-matrix}
which precisely describes the eigenvalues of the matrix $\vec{Q}$, their multiplicities
and the eigenspaces. We also allow the coordinates of the vector $\vec{y}$ to be in $[0,1]$,
not only in $(0,1)$.
\begin{theorem}[Spectral properties of a special matrix]
  \label{theorem:spectral-properties-of-Q}
  Let $\vec{y}\in\reals^C$ be a vector whose entries $y_j\in[0,1]$, 
  $\sum_{j=1}^Cy_j=1$, and let
  \[ \vec{Q}=\diag(\vec{y})-\vec{y}\,\vec{y}^\trans
  \]
  Then the eigenvalues of $\vec{Q}$ are non-negative.
  Let
  \[P=\{j\in\{1,2,\ldots,C\}\,:\, y_j\neq 0\}.\]
  We note that $P\neq\emptyset$.
  The eigenvalues of $\vec{Q}$ are as follows:
  \begin{enumerate}
  \item $0$ is an eigenvalue; the nullspace of $\vec{Q}$ is exactly the subspace of all vectors
    $\vec{z}$ such that $z_j=z$ is the same for all $j\in P$ and otherwise
    arbitrary. Hence, the multiplicity of $0$ as an eigenvalue is by one
    more than the number of times $0$ appears as a coordinate in $\vec{y}$.
  \item Every $y_j>0$ is an eigenvalue of $\vec{Q}$
    of multiplicity by one less than its multiplicity in the multi-set $\{y_j\}_{j=1}^C$. Thus, if the multiplicity of some $y_j>0$ is 1, then $y_j$ is not an eigenvalue.
  \item If the unique elements of the multi-set $\{y_j\}_{j=1}^C$ are $a_s$, $s=1,2,\ldots,r$,
    then there exists one simple eigenvalue in every interval $(a_s,a_{s+1})$ for
    $s=1,2,\ldots,r-1$.
  \end{enumerate}
\end{theorem}
\begin{proof}
Let $\vec{Q}\,\vec{z}=\lambda\,\vec{z}$. Then for
$j=1,2\ldots,C$:
\[ y_j\,z_j - y_j\,\vec{y}^\trans\vec{z} = \lambda\, z_j\]
Hence for every $j$:
\[(y_j-\lambda) z_j = y_j\langle\vec{y},\vec{z}\rangle.\]

\noindent\textbf{Case $\lambda = 0$:} Then $y_j(z_j-\langle\vec{y},\vec{z}\rangle) = 0$. Therefore
either $y_j = 0$ and $z_j$ is arbitrary, or
$z_j=\langle\vec{y},\vec{z}\rangle$, which is constant for $j\in P$. 

\noindent\textbf{Case $\lambda\neq 0$:} Let $c = \langle\vec{y},\vec{z}\rangle$. Hence for all $j$
\[(y_j-\lambda) z_j = c\,y_j.\]
Thus for all $j\notin P$
\[ -\lambda z_j = 0 \]
which implies $z_j=0$ for $j\notin P$.
Assume that for some $\ell\in P$ $\lambda= y_\ell$ is an eigenvalue. We have
\[0 = (y_\ell-\lambda) z_\ell = c\,y_\ell\]
and $y_\ell\neq 0$. Therefore $c = 0$ and
\[(y_j-\lambda)\,z_j = 0 \]
for all $j$. Therefore, $z_j=0$ for those $j$ that $y_j\neq \lambda$,
and $z_j$ is arbitrary for those $j$ that $y_j=\lambda=y_\ell$. In addition
we must have
\[ c = \langle\vec{y},\vec{z}\rangle = \sum_{j\in P} y_j\,z_j =
  \sum_{j\in P} \lambda z_j = \lambda \sum_{j\in P} z_j = 0\]
i.e. $\vec{z}\perp \mathbbm{1}$.
We can see that the dimension of the eigenspace is by one less than the number
of times $\lambda$ appears as a coordinate in vector $\vec{y}$. Thus, $\lambda=y_j$ is an eigenvalue if $y_j$ has multiplicity of at least two in the multi-set $\{y_j\}_{j=1}^C$. 

Let now $\lambda \neq y_\ell$ for all $\ell\in P$. Thus for all $j\in P$:
\[ z_j = c \frac{y_j}{y_j -\lambda}. \]
Therefore
\[ c = \langle\vec{y},\vec{z}\rangle = c \sum_{j\in P} \frac{y_j^2}{y_j-\lambda}. \]
As $c\neq 0$ (else $\vec{z}=0$), we have
\[ \sum_{j\in P} \frac{y_j^2}{y_j-\lambda} = 1. \]
Clearly, $\lambda=0$ is a root of this equation. For $\lambda < 0$ the left-hand side is $>1$.
Therefore, all other roots $\lambda$ are positive. 
Let $\{y_j\}_{j\in P} = \{a_1,a_2,\ldots a_r\}\subseteq (0,1)$, where $a_k$ repeats $\nu_k$ times.
Then
\[ f(\lambda):=\sum_{s=1}^r \frac{\nu_s\,a_s^2}{a_s-\lambda} = 1. \]
As $f'(\lambda) > 0$, the function $f(\lambda)$ is increasing on each
interval $(-\infty,a_1)$, $(a_1,a_2)$, $\ldots$, $(a_{r-1},a_r)$,
$(a_r,\infty)$. Every interval, except for the first and last one is
mapped to $(-\infty,\infty)$, hence, there is a root $\lambda$ in it,
yielding $r-1$ roots. Also $f( (-\infty,a_1) ) = (0,\infty)$, yielding
root $\lambda=0$ and $f((a_r,\infty)) = (-\infty, 0)$, yielding no
roots. Clearly, all real roots are simple. Also, $\lambda$ is a simple eigenvalue,
as it is clear the $\vec{z}$ is unique up to a constant. 
\end{proof}

\printbibliography

\typeout{get arXiv to do 4 passes: Label(s) may have changed. Rerun}
\end{document}